\documentclass[10pt]{article}
\usepackage[total={6in,9.5in}]{geometry}
\usepackage{amsmath,amssymb,amsthm}
\usepackage{siunitx}
\usepackage{optidef}
\usepackage{palatino}
\usepackage{graphicx,epstopdf,epsfig}
\usepackage[dvipsnames]{xcolor}
\usepackage{float}
\usepackage[caption=false,font=footnotesize]{subfig}
\usepackage{url}
\usepackage{microtype}
\usepackage[normalem]{ulem}
\usepackage[section]{placeins} 
\usepackage{centernot}
\usepackage{enumitem}
\usepackage[noadjust]{cite}

\newtheorem{theorem}{Theorem}
\newtheorem{proposition}{Proposition}
\newtheorem{definition}{Definition}
\newtheorem{corollary}{Corollary}
\newtheorem{lemma}{Lemma}

\newtheorem{remark}{Remark}

\newcommand{\Sy}{\mathbb{S}}
\renewcommand{\Re}{\mathbb{R}}

\newcommand{\Ur}{\mathbb{U}}
\newcommand{\Vr}{\mathbb{V}}
\newcommand{\Wr}{\mathbb{W}}
\newcommand{\Xr}{\mathbb{X}}
\newcommand{\Yr}{\mathbb{Y}}

\newcommand{\cN}{\mathcal{N}}

\def\A{\mathbf{A}}
\def\B{\mathbf{B}}
\def\C{\mathbf{C}}
\def\D{\mathbf{D}}

\def\F{\mathbf{F}}
\def\G{\mathbf{G}}

\def\I{\mathbf{I}}
\def\J{\mathbf{J}}
\def\K{\mathbf{K}}
\def\M{\mathbf{M}}
\def\N{\mathbf{N}}
\def\P{\mathbf{P}}
\def\Q{\mathbf{Q}}
\def\R{\mathbf{R}}
\def\S{\mathbf{S}}
\def\T{\mathbf{T}}

\def\V{\mathbf{V}}
\def\W{\mathbf{W}}

\def\b{\boldsymbol{b}}

\def\e{\boldsymbol{e}}

\def\h{\boldsymbol{h}}

\def\q{\boldsymbol{q}}
\def\r{\boldsymbol{r}}
\def\s{\boldsymbol{s}}
\def\u{\boldsymbol{u}}

\def\w{\boldsymbol{w}}
\def\x{\boldsymbol{x}}
\def\y{\boldsymbol{y}}
\def\z{\boldsymbol{z}}
\def\ze{\mathbf{0}}

\def\bxi{\boldsymbol{\xi}}

\newcommand{\prox}{\mathrm{prox}}
\newcommand{\Span}{\mathrm{span}}
\DeclareMathOperator*{\argmin}{\arg\!\min}
\DeclareMathOperator*{\diag}{diag}
\DeclareMathOperator*{\minimize}{minimize}

\newcommand\xqed[1]{%
  \leavevmode\unskip\penalty9999 \hbox{}\nobreak\hfill
  \quad\hbox{#1}}
\newcommand{\fd}{\xqed{$\blacklozenge$}} %remark
\begin{document}

\date{}

\title{On the Contractivity of Plug-and-Play Operators}

\author{C.~D.~Athalye, K.~N.~Chaudhury, and B.~Kumar\thanks{C.~D.~Athalye is with the Department of Electrical \& Electronics Engineering, BITS Pilani, K.~K.~Birala Goa Campus, Goa: 403726, India. K.~N.~Chaudhury and B.~Kumar are with the Department of Electrical Engineering, Indian Institute of Science, Bengaluru: 560012, India. 
During this research work, C.~D.~Athalye was with the Department of Electrical Engineering, Indian Institute of Science, Bengaluru: 560012, India. 
Email: chirayua@goa.bits-pilani.ac.in, kunal@iisc.ac.in, bhartenduk@iisc.ac.in. The work of C.~D.~Athalye was supported by the Department of Science and Technology, Government of India, under Grant IFA17-ENG227, and K.~N.~Chaudhury was supported by research grants CRG/2020/000527 and STR/2021/000011 from the Science and Engineering Research Board, Government of India.}
}

\maketitle

\begin{abstract}
In plug-and-play (PnP) regularization, the proximal operator in algorithms such as ISTA and ADMM is replaced by a powerful denoiser. 
This formal substitution works surprisingly well in practice. In fact, PnP has been shown to give state-of-the-art results for various imaging applications. 
The empirical success of PnP has motivated researchers to understand its theoretical underpinnings and, in particular, its convergence. 
It was shown in prior work that for kernel denoisers such as the nonlocal means, PnP-ISTA provably converges under some strong assumptions on the forward model. 
The present work is motivated by the following questions: Can we relax the assumptions on the forward model? Can the convergence analysis be extended to PnP-ADMM? Can we estimate the convergence rate? 
In this letter, we resolve these questions using the contraction mapping theorem: (i) for symmetric denoisers, we show that (under mild conditions) PnP-ISTA and PnP-ADMM exhibit linear convergence; and (ii) for kernel denoisers, we show that PnP-ISTA and PnP-ADMM converge linearly for image inpainting.
We validate our theoretical findings using reconstruction experiments. 
\end{abstract}
 
%\begin{keywords}
%Plug-and-Play, symmetric denoisers, kernel denoisers, ISTA, ADMM, contraction mapping, convergence.
%\end{keywords}

\section{Introduction}

Image reconstruction tasks such as denoising, inpainting, deblurring, and superresolution can be modeled as a linear inverse problem: we wish to recover an image $\bxi \in \Re^n$ from noisy linear measurements $\b = \A \bxi + \w$,            
where $\A \in \Re^{m \times n}$ is the forward model and $\w \in \Re^m$ is white Gaussian noise. 
A standard approach is to solve the optimization problem
\begin{equation}
\label{eq:opt-prob}
\minimize_{\x \in \Re^n} \, f(\x) +  g(\x), \quad  f(\x)= \frac{1}{2} \|\A\x - \b\|_2^2,
\end{equation}
where the loss function $f$ is derived from the forward model and $g$ is an image regularizer \cite{ribes2008linear,bouman2022foundations}. 
The choice of regularizer has evolved from simple Tikhonov and Laplacian regularizers \cite{chan2005image} to wavelet, total-variation, dictionary, etc. \cite{rudin1992nonlinear,zoran2011learning,chan2005image}, and to more recent learning-based models \cite{zhang2017learning,jin2017deep,zhang2017beyond}.

It has been shown that powerful denoisers can also be used for image regularization \cite{romano2017little,sreehari2016plug,ahmad2020plug,teodoro2019image}.
One such method is Plug-and-Play (PnP) regularization, where we fix a proximal algorithm for solving \eqref{eq:opt-prob} such as ISTA or ADMM \cite{bauschke2017convex} and formally replace the proximal operator of $g$ with a denoiser \cite{sreehari2016plug,Teodoro2019PnPfusion,gavaskar2020plug,gavaskar2021plug}. 
%(Iterative Shrinkage-Thresholding Algorithm) (Alternating Direction Method of Multipliers) 
It is not apparent if this ad-hoc replacement can work in the first place, but PnP has been shown to work surprisingly well for many imaging and signal processing applications \cite{sreehari2016plug,zhang2017learning}, \cite{CWE2017,yazaki2019interpolation,gavaskar2023PnPCS,nagahama2021graph,kamilov2023PnPCI}. 
In fact, it has been shown that simple kernel denoisers such as nonlocal means \cite{buades2005non} can yield state-of-the-art results which are comparable with deep learning methods \cite{nair2022plug}.  
The empirical success of PnP has sparked interest in theoretical questions stemming from the algorithm. The most basic issue is convergence since the PnP iterations are not strictly derived within an optimization framework. In the absence of a convergence guarantee, the iterations can possibly diverge and result in a poor reconstruction.
Convergence of PnP has been studied in several works; we refer readers to \cite{gavaskar2021plug,cohen2021regularization} and the references therein for a comprehensive survey. 
In particular, PnP convergence for linear inverse problems is studied in \cite{sreehari2016plug,Teodoro2019PnPfusion,gavaskar2020plug,gavaskar2021plug}, \cite{cohen2021regularization,nair2021fixed,xu2020mmse,Ryu2019_PnP_trained_conv,tirer2018image,Dong2018_DNN_prior}.

In this letter, we revisit the convergence question for two popular PnP algorithms, PnP-ISTA and PnP-ADMM \cite{sreehari2016plug, Ryu2019_PnP_trained_conv,gavaskar2020plug}. 
To keep the analysis tractable, we restrict ourselves to linear inverse problems and linear denoisers  \cite{sreehari2016plug, Teodoro2019PnPfusion,gavaskar2021plug,nair2022plug}. 
Specifically, we show that the convergence results in \cite{gavaskar2020plug} can be extended and strengthened in various directions. Similar to \cite{gavaskar2020plug}, we exploit the linearity of the denoiser and the fact that PnP-ISTA and PnP-ADMM can be expressed as an iterative map $\x \mapsto \Q\x+\r$; however, our analysis is fundamentally different from \cite{gavaskar2020plug}. 
Our main results are summarized below:

\begin{enumerate}[label=(\roman*)]
\item For symmetric denoisers \cite{sreehari2016plug, Teodoro2019PnPfusion}, we give necessary and sufficient conditions for PnP-ISTA and PnP-ADMM to be contractive w.r.t. the spectral norm (Theorems \ref{thm:PnP-ISTA_contractivity} and \ref{thm:PnP-ADMM_contractivity}). 
Subsequently, using the contraction mapping theorem \cite{rudin1976principles}, we establish linear convergence of these PnP methods for applications such as inpainting, deblurring, and superresolution (Corollaries \ref{cor:PnP-ISTA-applications} and \ref{cor:PnP-ADMM_linear-convergence}).
To our knowledge, these are the first results on {\em linear convergence} of PnP, where the assumptions are practically verifiable.

\item For image inpainting using (nonsymmetric) kernel denoisers \cite{buades2005non}, we prove that $\|\Q\|_\D <1$ for both PnP-ISTA and PnP-ADMM (Theorem \ref{thm:PnP-kernel}), where $\|\cdot\|_\D$ is an appropriate matrix norm. It is shown in \cite{gavaskar2020plug} that the spectral radius $\varrho(\Q) <1$  for PnP-ISTA based on kernel denoisers, which is sufficient to guarantee convergence. 
However, the condition $\|\Q\|_\D <1$ allows us to establish linear convergence for PnP-ISTA (Theorem \ref{thm:PnP-kernel}), which is not guaranteed by $\varrho(\Q) <1$.
\end{enumerate}

This letter is organized as follows. We cover the necessary background in Section \ref{sec:background}. 
We state and discuss our main results in Section \ref{sec:results} and numerically validate them in Section \ref{sec:experiments}; the proofs are deferred to Appendix. 
%Finally, we conclude in Section \ref{sec:conc}.

\section{Background}
\label{sec:background}

In this section, we briefly cover some background material.
In the classical ISTA algorithm for solving \eqref{eq:opt-prob}, we start with an initialization $\x_0 \in \Re^n$ and perform the updates $\x_{k+1} = \prox_{\gamma g} \big(\x_k- \gamma \nabla \! f(\x_k) \big)$ for $k \geqslant 0$,
where $\prox_{g}\colon \Re^n \to \Re^n$ is the proximal operator of $g$, 
\begin{equation}
\label{eq:prox}
\prox_g(\x) := \argmin_{\z \in \Re^n} \left(\frac{1}{2} \|\x - \z\|_2^2 + g(\z)\right),
\end{equation}
and $\gamma > 0$ is the step size \cite{beck2009fast}. 
In PnP regularization, instead of going through $g$ and its proximal operator, we directly substitute $\prox_{\gamma g}$ by a denoiser. 
In this work, we consider linear denoisers and a quadratic loss function $f$ (see \eqref{eq:opt-prob}), i.e., we perform the updates as
\begin{equation}
\label{eq:pnp-ista}
\x_{k+1} = \W\big(\x_k- \gamma (\A^\top\!  \A\x_k - \A^\top \b)\big), 
\end{equation}
where $\W$ is a linear denoiser. Examples of such denoisers are Yaroslavsky \cite{yaroslavsky1985digital}, bilateral \cite{tomasi1998bilateral}, nonlocal means \cite{buades2005non,sreehari2016plug}, LARK \cite{takeda2007kernel}, GLIDE \cite{talebi2013global}, and GMM \cite{Teodoro2019PnPfusion}. We refer to \eqref{eq:pnp-ista} as PnP-ISTA. 
The above idea can be applied to other iterative algorithms such as PnP-ADMM \cite{sreehari2016plug,CWE2017}, which is based on the ADMM algorithm \cite{bauschke2017convex}. In PnP-ADMM, starting with $\y_0, \z_0 \in \Re^n$ and some fixed $\rho > 0$, the updates are performed as follows:
\begin{align}
\label{eq:pnp-admm}
\x_{k+1} &= \prox_{\rho f} (\y_k - \z_k),  \nonumber\\ 
\y_{k+1} &= \W (\x_{k+1} + \z_k),  \\ 
\z_{k+1} &= \z_k + \x_{k+1} - \y_{k+1}. \nonumber
\end{align}

In prior work \cite{gavaskar2020plug}, it was shown that PnP-ISTA converges if $\A$ and $\W$ satisfy some assumptions and $\gamma$ is small. 
More specifically, it was shown that (under some conditions on $\W$) convergence is guaranteed for inpainting if $\gamma \in (0,1)$. The analysis is based on the observation that \eqref{eq:pnp-ista} can be written as
\begin{equation}
\label{eq:PnP-iterates}
\x_{k+1} = \P\x_k + \q,
\end{equation}
where $\P= \W(\I-\gamma \A^\top\!  \A)$ and $ \q =\gamma \W\A^\top \b$. In particular, it was shown in \cite{gavaskar2020plug} that if at least one pixel is sampled from the averaging window (of the denoiser) around each pixel, then  $\varrho(\P) < 1$ for inpainting; 
subsequently, it follows that the sequence $\{\x_k\}$ in \eqref{eq:PnP-iterates} converges for an arbitrary initialization $\x_0$ \cite{meyer2000matrix}.
However, we show that the above assumption in \cite{gavaskar2020plug} can be dropped if we adopt a different analysis.

Similar to PnP-ISTA, we can express \eqref{eq:pnp-admm} in terms of a linear dynamical system. More precisely, we have the following observation by adapting the results in \cite{Ryu2019_PnP_trained_conv,nair2021fixed}. 
\begin{proposition}
\label{prop:DS_PnP-ADMM}
Let $\{(\x_k, \y_k, \z_k)\}_{k \geqslant 1}$ be the iterates generated by PnP-ADMM starting from $\y_0, \z_0 \in \Re^n$. Define the sequence $\{\u_k\}_{k \geqslant 1}$ as follows:  $\u_1 = \y_1 + \z_1$, and
\begin{equation}
\label{eq:uk}
\u_{k+1} = \frac{1}{2} \u_k + \frac{1}{2} (2\prox_{\rho f}-\I)(2\W-\I)(\u_k)
\end{equation}
for $k \geqslant 1$. Then $\y_k = \W\u_k$ for $k \geqslant 1$. 
\end{proposition}

The above result can be traced to the observation that ADMM can be interpreted as Douglas-Rachford splitting applied to the dual problem of \eqref{eq:opt-prob}, e.g., see \cite{bauschke2017convex}.
For problem \eqref{eq:opt-prob}, a direct calculation based on \eqref{eq:prox} gives
\begin{equation}
\label{eq:proxf}
\prox_{\rho f} (\x) = \big(\I+\rho \A^\top\!  \A\big)^{-1} \x + \h,
\end{equation}
where $\h \in \Re^n$ is an irrelevant parameter. 
Substituting \eqref{eq:proxf} in \eqref{eq:uk}, we obtain $\u_{k+1}=\R\u_k + \s$, where
\begin{subequations}
\label{eq:def-R}
\begin{equation}
\R =  \frac{1}{2}(\I + \J), \qquad \J = \F\V,
\end{equation}
\begin{equation}
\label{eq:def-VF}
\F = 2(\I + \rho \A^\top\A )^{-1} - \I,\qquad \V = (2\W-\I), 
\end{equation}
\end{subequations}
and $\s \in \Re^n$ is an irrelevant parameter. 

Our convergence analysis is built on the contraction mapping theorem \cite{rudin1976principles}. In particular, we use the following result.
\begin{proposition}
\label{prop:linear-convergence}
Let $\x_0 \in\Re^n$ and $\{\x_k\}_{k \geqslant 1}$ be the sequence generated as $\x_{k+1} = \Q\x_k + \r$. Then $\{\x_k\}$ converges linearly if $\Q$ is a contraction, i.e., if there exists a norm $\|\cdot\|$ on $\Re^n$ such that $\|\Q\|:= \max \,\{ \, \|\Q\x\| : \, \|\x\| =1\}< 1$.
\end{proposition}
%%\begin{proof}
%%Since $\|\Q\| < 1$, by the contraction mapping theorem, the sequence $\{\x_k\}$ converges to the fixed point, say $\x^*$, of the affine map $\Q(\cdot) + \r$. Therefore, 
%%\begin{equation*}
%%\x_{k+1} - \x^* = (\Q\x_k + \r) - (\Q\x^*+\r);
%%\end{equation*}
%%subsequently, $\|\x_{k+1} - \x^*\| < \|\Q\|\,\|\x_k - \x^*\|$.
%%\end{proof}

Recall that a sequence $\{\x_k\}$ is said to converge {\em linearly} to $\x^*$ if there exists $\theta \in (0,1)$ such that $\|\x_{k+1} - \x^*\| \leqslant\theta \, \|\x_k - \x^*\| $ for all $k$ \cite{bauschke2017convex,boyd2004convex}. More generally, we use the term ``linear convergence'' if there exist $c>0$ and $\theta \in (0,1)$ such that  $\|\x_{k} - \x^*\| \leqslant  c \, \theta^k$ for all $k$ \cite{bonnans2006numerical}. 

\section{Results and Discussion}
\label{sec:results}
In this section, we motivate and discuss our main results; the proofs are deferred to Appendix.
We first consider linear symmetric denoisers and then nonsymmetric kernel denoisers; the analysis and results are slightly different in each case.   

We use $\Sy^n$ to denote the space of $(n \times n)$ symmetric matrices. 
We denote the spectral radius of $\Q \in \Re^{n \times n}$ by $\varrho(\Q)$.
A matrix $\C$ is said to be nonnegative if all its entries are nonnegative.
We use $\e$ to denote the vector of all-ones, $\I$ for the identity matrix, and $\cN(\C)$ for the null space of $\C$. 
A matrix $\C$ is called stochastic if $\C$ is nonnegative and $\C\e=\e$. 

\subsection{Symmetric Denoisers}
\label{subsec:symmetric}

%We first make precise what is meant by a symmetric denoiser.
\begin{definition}
\label{def:symmetric-denoiser}
A linear denoiser $\W$ is said to be symmetric if $\W$ is stochastic and symmetric positive semidefinite. 
\end{definition}

Symmetric denoisers are discussed in \cite{milanfar2013symmetrizing,sreehari2016plug,Teodoro2019PnPfusion}. 
Note that the eigenvalues of a symmetric denoiser are in the interval $[0,1]$.
The above-cited instances of symmetric denoisers are in fact irreducible; subsequently, it follows from the Perron-Frobenius theorem \cite{meyer2000matrix} that $1$ is a simple eigenvalue of $\W$, i.e., its algebraic multiplicity is one. 

To motivate the main idea, we consider the PnP-ISTA operator $\P$ in \eqref{eq:PnP-iterates} and write it as
\begin{equation}
\label{eq:def-PG}
\P = \W\G, \quad \G = \I-\gamma \A^\top\A.
\end{equation}
Since $\W$ is a symmetric denoiser, $\|\W\|_2 = \varrho(\W) = 1$. On the other hand, if $0 < \gamma < 2/\varrho(\A^\top\! \A)$, then $\|\G\|_2 \leqslant1$. 
%In fact, when $\A$ has a nontrivial null space, $\|\G\|_2 = 1$.
Since neither $\W$ nor $\G$ is contractive (w.r.t. to the Euclidean norm), we cannot immediately conclude that $\P=\W \G$ is a contraction. 
However, we observed in numerical experiments that in the case of symmetric denoisers, $\|\P\|_2 < 1$ if $0 < \gamma < 2/\varrho(\A^\top\! \A)$.
This raises the following question: when is the composition of two symmetric nonexpansive operators contractive? An equivalent condition for the same is given next, which is a core result of this letter. 
\begin{lemma}
\label{lemma:contractivity}
Suppose $\M,\N \in \Sy^n$ are such that $\|\M\|_2 \leqslant1$ and $\|\N\|_2 \leqslant1$. 
Let $\Ur$ and $\Vr$ be the subspaces spanned by the eigenvectors of $\M$ and $\N$ corresponding to eigenvalues $\pm 1$. 
Then $\|\M\N\|_2 < 1$ if and only if $\Ur \cap \Vr = \{\ze\}$.
\end{lemma}

Using Lemma \ref{lemma:contractivity}, we can explain our empirical observation that $\|\P\|_2 < 1$.

\begin{theorem}
\label{thm:PnP-ISTA_contractivity}
Let $\W$ be a symmetric denoiser, and let $0 < \gamma < 2/\varrho(\A^\top\!  \A)$. Then $\|\P\|_2 < 1$ if and only if $\cN(\I-\W) \cap \cN(\A) = \{\ze\}$.
\end{theorem}

As explained next, we get the following result from Theorem \ref{thm:PnP-ISTA_contractivity} and Proposition \ref{prop:linear-convergence}. 
\begin{corollary}
\label{cor:PnP-ISTA-applications}
Let $\W$ be a symmetric denoiser with $1$ as a simple eigenvalue and let $\gamma \in (0,2)$. Then PnP-ISTA is contractive and converges linearly for inpainting, deblurring, and superresolution.
\end{corollary}

Following Theorem \ref{thm:PnP-ISTA_contractivity}, we need to check that $\varrho(\A^\top\!  \A) \leqslant 1$ and $\cN(\I-\W) \, \cap \, \cN(\A) = \{\ze\}$. This requires us to examine $\A$ for the applications under consideration.
In inpainting, we observe a subset of the pixels from the original image \cite{ribes2008linear}. If we represent the image by a vector in $\Re^n$, where $n$ is the number of pixels, and let $\Omega \subseteq \{1,2,\ldots,n\}$ be the observed pixels, then $\A$ for inpainting is the diagonal matrix 
\begin{equation}
\label{eq:A-inpaint}
\A_{ii} = 
\begin{cases}
1, \ \   &  i \in \Omega, \\
0, \ \ &  i \notin \Omega.
\end{cases}
\end{equation}
We assume that $\Omega \neq \emptyset$, which holds trivially in practice.
Thus, $\varrho(\A^\top\!  \A) = 1$ for inpainting.
For deblurring and superresolution, $\A = \B$ and $\A=\S\B$, respectively, where $\B \in \Re^{n \times n}$ is a blurring matrix corresponding to a lowpass filter and $\S \in \Re^{m \times n}$ is a subsampling matrix. 
In practice, a blur matrix $\B$ is circulant and stochastic \cite{ribes2008linear}, whereby $\|\B\|_2 = 1$.
Therefore, $\varrho(\A^\top\!  \A) =1$ for deblurring. As for superresolution, since $\|\S\|_2 = \|\B\|_2 = 1$, we have $\varrho(\A^\top\!  \A) = \|\A\|_2^2 \leqslant 1$. 
Thus, the range of $\gamma$ in Theorem \ref{thm:PnP-ISTA_contractivity} can be $(0, 2)$ for inpainting, deblurring, and superresolution.
 
Next, we argue that $\cN(\I-\W) \, \cap \, \cN(\A) = \{\ze\}$ in Corollary \ref{cor:PnP-ISTA-applications}. Since $1$ is assumed to be a simple eigenvalue of $\W$, we have $\cN(\I-\W) = \Span\{\e\}$. Note that $\A\e \neq 0$ for inpainting, deblurring, and superresolution. 
Therefore, for the applications under consideration in Corollary \ref{cor:PnP-ISTA-applications}, using Theorem \ref{thm:PnP-ISTA_contractivity}, we have $\|\P\|_2 < 1$ for $\gamma \in (0, 2)$. Subsequently, Corollary \ref{cor:PnP-ISTA-applications} follows from Proposition \ref{prop:linear-convergence}.
 
The contractivity of PnP-ISTA was established for deep denoisers in \cite{Ryu2019_PnP_trained_conv}, however, under a strong assumption that $f$ is strongly convex. 
Note that our analysis applies to a broader class of linear inverse problems such as inpainting, deblurring and superresolution for which $f$ is not strongly convex.
Also, the present assumption for PnP-ISTA convergence is significantly milder than that in \cite{gavaskar2020plug}. 
Recall that $1$ is indeed a simple eigenvalue of symmetric denoisers discussed in \cite{milanfar2013symmetrizing,sreehari2016plug,Teodoro2019PnPfusion}.
%(see the discussion after Definition \ref{def:symmetric-denoiser})
Moreover, the convergence interval $\gamma \in (0,2)$ in Corollary \ref{cor:PnP-ISTA-applications} is similar to that for classical ISTA \cite{beck2009fast}.

%\begin{remark}
%Notice that if $\gamma = 0$, then $\G=\I$; therefore, $\P=\W\G$ violates the equivalent condition in Lemma \ref{lemma:contractivity}. 
%\end{remark}

We next look at PnP-ADMM with a symmetric denoiser. The core result in this regard is the following.
\begin{theorem}
\label{thm:PnP-ADMM_contractivity}
Let $\W$ be a symmetric denoiser and $\rho > 0$. Then $\|\R\|_2 < 1$ if and only if $\cN(\I-\W) \cap \cN(\A) = \{\ze\}$.
\end{theorem}

As explained earlier, the condition in Theorems \ref{thm:PnP-ISTA_contractivity} and \ref{thm:PnP-ADMM_contractivity} is satisfied for inpainting, deblurring, and superresolution if $1$ is a simple eigenvalue of $\W$. Thus, as a consequence of Theorem \ref{thm:PnP-ADMM_contractivity}, we can conclude the following (see Remark \ref{rmk:PnP-ADMM_linear-convergence} in Appendix).
\begin{corollary}
\label{cor:PnP-ADMM_linear-convergence}
Let $\W$ be a symmetric denoiser for which $1$ is a simple eigenvalue. Then, for any $\rho > 0$, PnP-ADMM is contractive and converges linearly for inpainting, deblurring, and superresolution.
\end{corollary}

The convergence interval $\rho \in (0,\infty)$ in Theorem \ref{thm:PnP-ADMM_contractivity} is similar to standard convergence results for ADMM \cite{bauschke2017convex}.

\subsection{Kernel Denoisers}
\label{subsec:kernel}

We now turn to kernel denoisers \cite{yaroslavsky1985digital,lee1983digital,tomasi1998bilateral,buades2005non,sreehari2016plug,takeda2007kernel}. These are typically nonsymmetric, whereby Theorems \ref{thm:PnP-ISTA_contractivity} and \ref{thm:PnP-ADMM_contractivity} do not apply. First, we make precise what we mean by a kernel denoiser.
\begin{definition}
\label{def:kernel-denoiser}
Let $\K$ be nonnegative and symmetric positive semidefinite, and let $\D=\diag(\K\e)$. We call $\K$ a kernel matrix and $\W = \D^{-1} \K$ a kernel denoiser. We assume that $\K\e$ is a positive vector and $1$ is a simple eigenvalue of $\W$.
\end{definition}

The last two assumptions are met for most kernel denoisers where $\W$ is nonnegative and irreducible \cite{milanfar2013tour}.
In the archetype kernel denoiser, such as nonlocal means \cite{buades2005non}, $\K$ is derived from a Gaussian kernel acting on image patches. The precise definition of $\K$ is not relevant to our discussion; instead, we refer the reader to \cite{milanfar2013tour,milanfar2013symmetrizing} for precise definitions. We just need the following abstract property of a kernel denoiser (e.g., see \cite{gavaskar2021plug}): $\W$ is similar to a symmetric positive semidefinite matrix, and its eigenvalues are in $[0,1]$.  

It turns out that we generally cannot ensure that $\|\P\|_2 < 1$ or $\|\R\|_2 < 1$ for kernel denoisers (see Table \ref{tab:counter_examples}). However, for the inpainting problem, we can show that $\P$ is a contraction with respect to $\|\cdot\|_\D$, where $\D$ is as in Definition \ref{def:kernel-denoiser}. 

\begin{remark}
\label{rmk:D-norm}
The inner product and norm on $\Re^n$ w.r.t. a symmetric positive definite matrix $\D$ are defined as follows: 
\begin{equation*}
\label{eq:D-norm}
\langle \x,\y \rangle_\D := \x^\top \D \y, \qquad \|\x\|_\D := \langle \x, \x \rangle_\D^{1/2},
\end{equation*}
and $\|\Q\|_\D := \max \, \{ \|\Q\x\|_\D : \|\x\|_\D=1\}$ for $\Q \in \Re^{n \times n}$. \fd 
\end{remark}
\begin{theorem}
\label{thm:PnP-kernel}
Let $\W$ be a kernel denoiser. Consider the inpainting problem where $\A$ is given by \eqref{eq:A-inpaint}.  
If $0 < \gamma < 2$, then $\|\P\|_\D < 1$; consequently, PnP-ISTA converges linearly. 
On the other hand, if $\W$ is nonsingular and $\rho > 0$, then $\|\R\|_\D < 1$; thus, PnP-ADMM converges linearly. 
\end{theorem}

We note that the nonsingularity assumption of $\W$ in Theorem \ref{thm:PnP-kernel} holds for nonlocal means \cite{nair2022plug}. 
Also, note that the use of $\|\cdot\|_\D$ in Theorem \ref{thm:PnP-kernel} is crucial in guaranteeing contractivity. In fact, as reported in Table \ref{tab:counter_examples}, we have observed in numerical experiments that $\|\P\|_2$ and $\|\R\|_2$ can indeed exceed $1$.

\begin{table}[htbp]
\caption{Inpainting of the Barbara image with $0.1\%$ samples: norms of PnP-ISTA and PnP-ADMM operators with the NLM denoiser.}
\begin{center}
\begin{tabular}{|c|c|c|c|c|}
\hline
parameter & \multicolumn{2}{c|}{PnP-ISTA} & \multicolumn{2}{c|}{PnP-ADMM} \\
\cline{2-3} \cline{4-5}
\textbf{$(\gamma / \rho)$} & \textbf{\textit{$\|\P\|_2$}}& \textbf{\textit{$\|\P\|_\D$ }}& \textbf{\textit{$\|\R\|_2$}} & \textbf{\textit{$\|\R\|_\D$}}  \\ \hline
0.25 & 1.0167 & 0.9998 & 1.0168 & 0.9998 \\ \hline
0.50 & 1.0165 & 0.9996 & 1.0163 & 0.9999 \\ \hline
0.75 & 1.0164 & 0.9996 & 1.0165 & 0.9998 \\ \hline
\end{tabular}
%\vspace{-4mm}
\end{center}
\label{tab:counter_examples}
\end{table}

\begin{remark}
It follows from Remark \ref{rmk:D-norm} that $\|\P\|_\D$ is the largest singular value of $\D^\frac{1}{2} \P \D^{-\frac{1}{2}}$. Thus, we can compute $\|\P\|_\D$ and $\|\R\|_\D$ using the power method \cite{matrixcomputations}. \fd
%For numerically efficient computations, it helps to treat $\P$ and $\R$ in terms of operators involving $\W$ and $\A$ (see \eqref{eq:def-R} and \eqref{eq:def-PG}).
\end{remark}

\section{Numerical Results}
\label{sec:experiments}

We numerically validate our contractivity and linear convergence results for three restoration applications: inpainting, deblurring, and superresolution. 
For all the experiments, we have made some fixed choices. We have used nonlocal means as a kernel denoiser \cite{gavaskar2021plug}. 
For symmetric denoisers, we have used DSG-NLM \cite{sreehari2016plug} (a form of symmetrized NLM) and the GMM denoiser \cite{Teodoro2019PnPfusion}.  
%The denoisers are derived from the Barbara image. 
We have randomly sampled $30\%$ pixels for inpainting; the blur matrix $\B$ is a uniform $11 \times 11$ blur for deblurring and superresolution, and $\S$ is uniform $2\mbox{x}$-subsampling for superresolution. 
The codes of our experiments are available at https://github.com/Bhartendu-Kumar/PnP-Conv.

\begin{figure}
\centering
\includegraphics[width=0.47\linewidth]{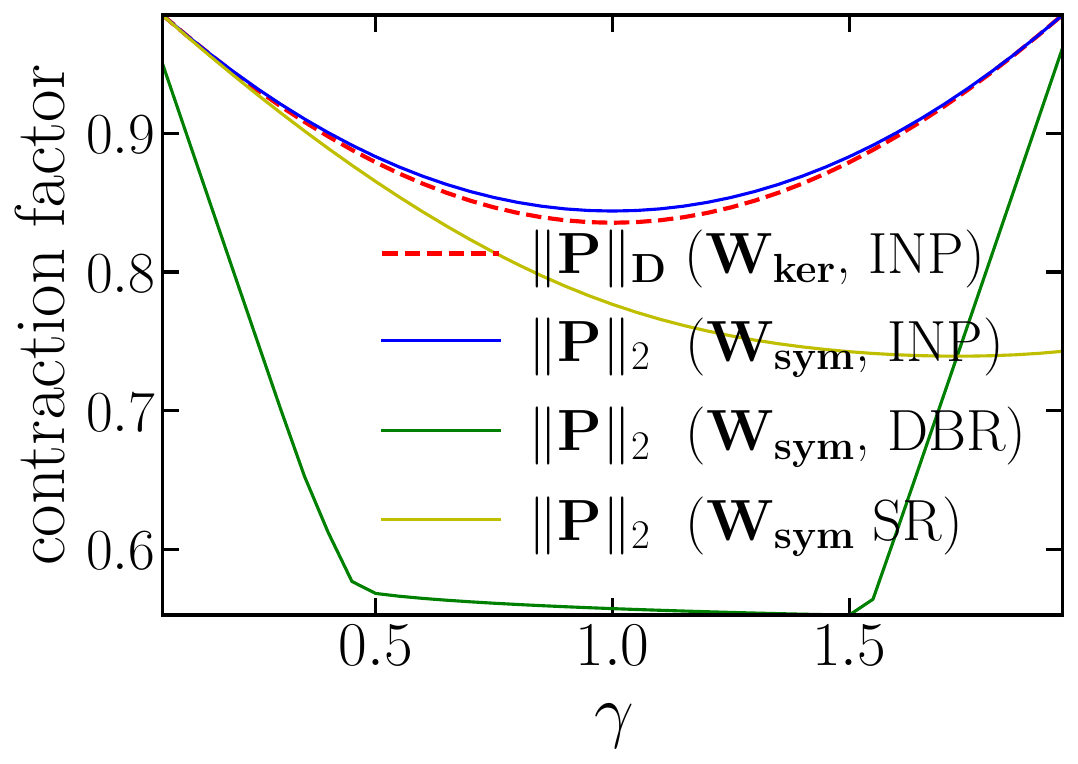} 
\includegraphics[width=0.47\linewidth]{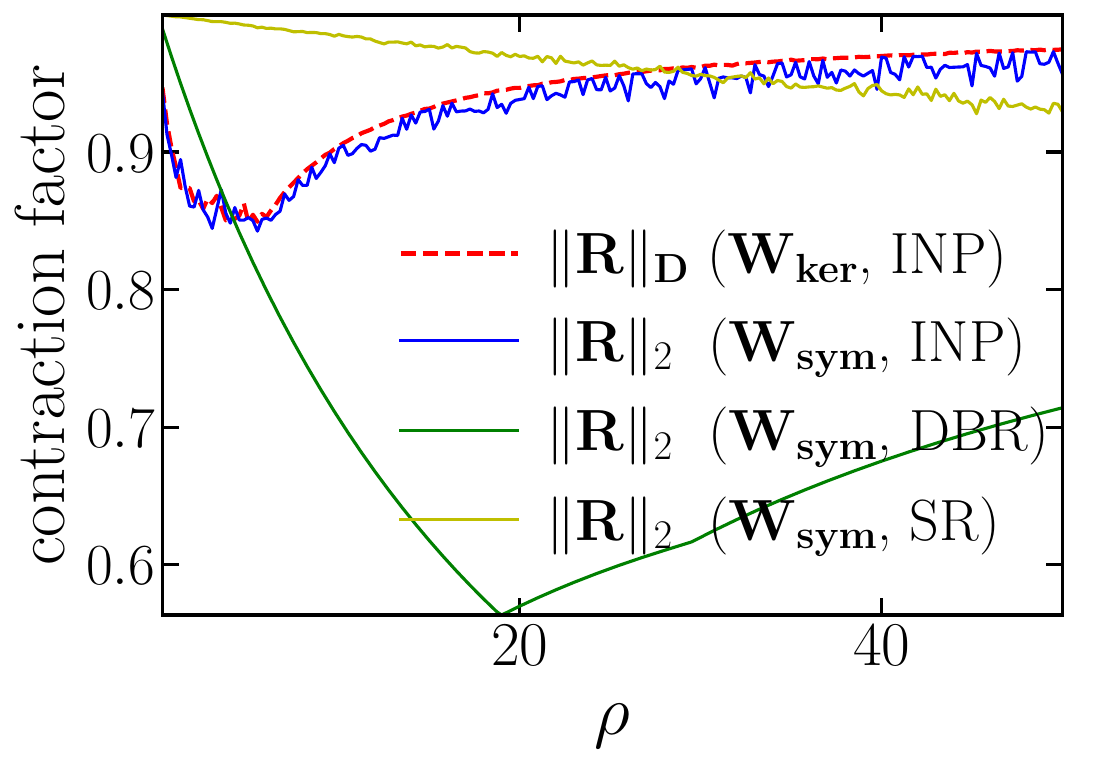}
\caption{Contraction factors for PnP-ISTA (left) and PnP-ADMM (right) for inpainting (INP), deblurring (DBR), and superresolution (SR); $\W_{\mbox{sym}}$ is the symmetric DSG-NLM denoiser \cite{sreehari2016plug} and $\W_{\mbox{ker}}$ is the NLM denoiser \cite{gavaskar2021plug}.}
\label{fig:plots}
\end{figure}

\begin{figure}
\centering
\includegraphics[width=1.0\linewidth]{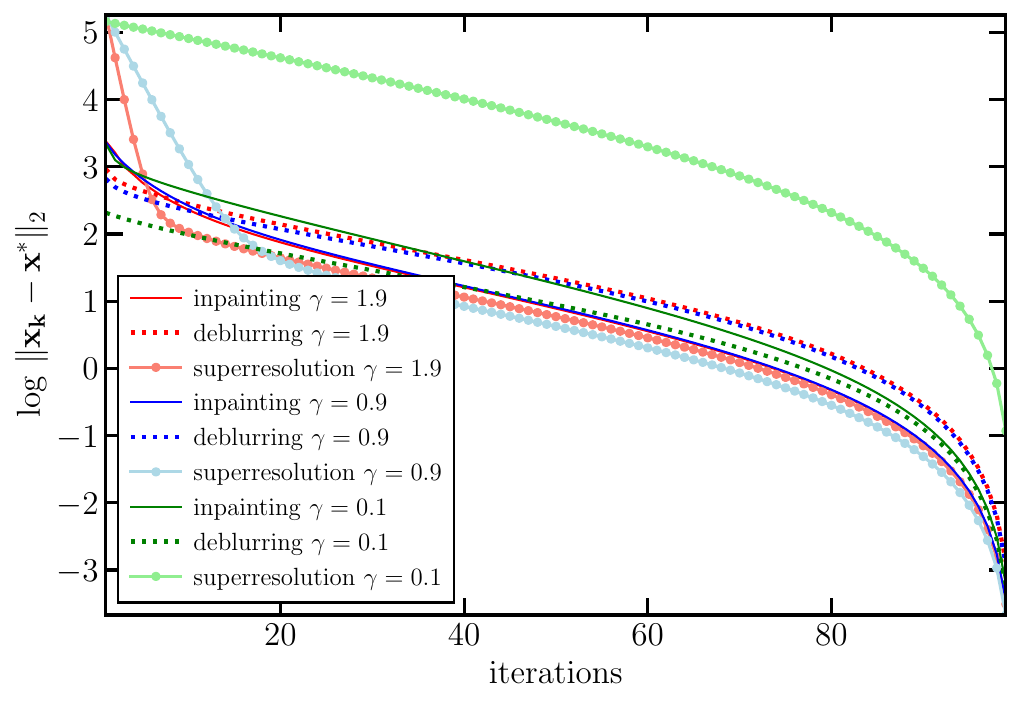} 
\caption{PnP-ISTA iterates for inpainting, deblurring, and superresolution, using the symmetric DSG-NLM denoiser \cite{sreehari2016plug} and for different step size $\gamma$; $\x^*$ is the iterate after 100 iterations. }
\label{fig:convergence}
\end{figure}

\begin{figure}
\centering
\subfloat[Observed]{\includegraphics[width=0.32\linewidth]{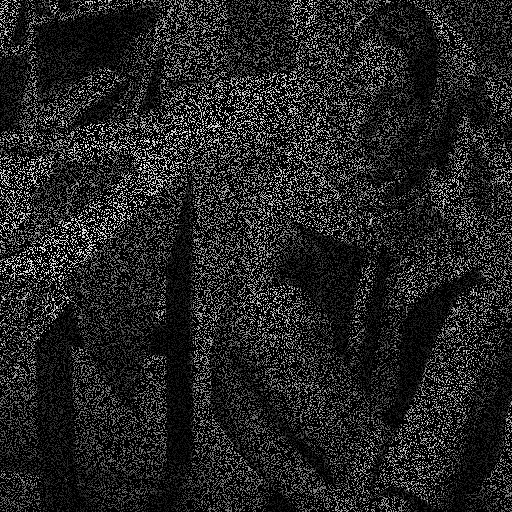}} \hspace{0.1mm}
\subfloat[PnP-ISTA]{\includegraphics[width=0.32\linewidth]{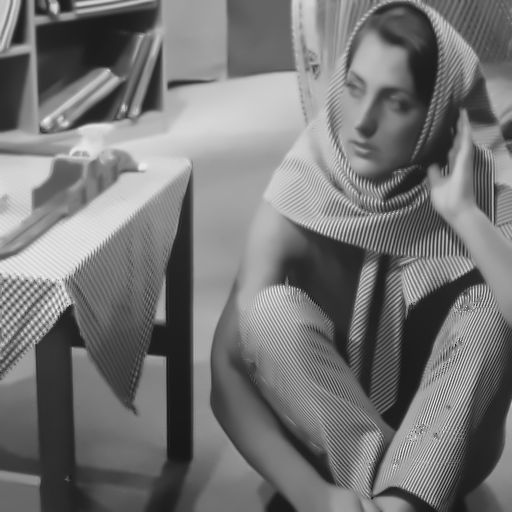}} \hspace{0.1mm}
\subfloat[PnP-ADMM.]{\includegraphics[width=0.32\linewidth]{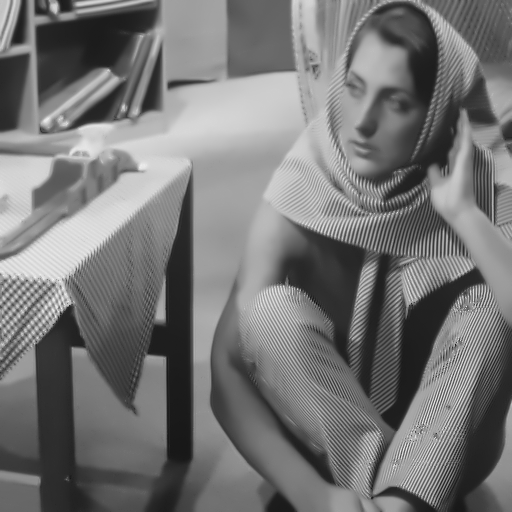}}
\vspace{2mm}
\subfloat[Observed]{\includegraphics[width=0.32\linewidth]{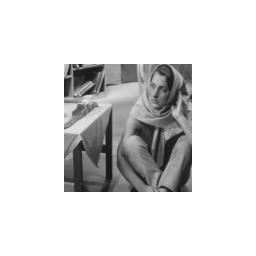}} \hspace{0.1mm}
\subfloat[PnP-ISTA]{\includegraphics[width=0.32\linewidth]{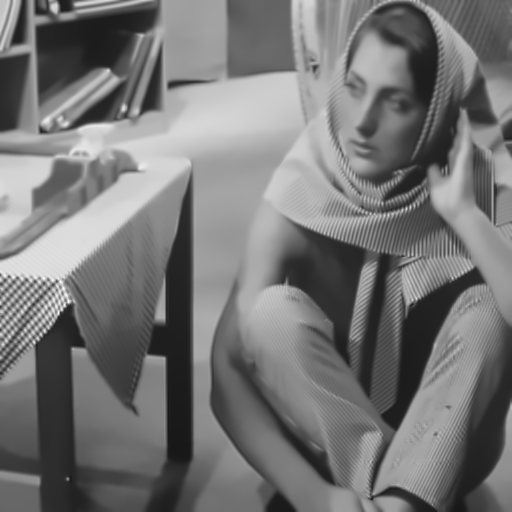}} \hspace{0.1mm}
\subfloat[PnP-ADMM]{\includegraphics[width=0.32\linewidth]{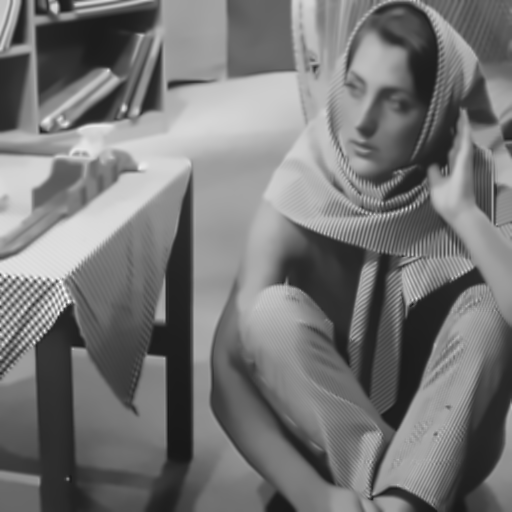}}
\caption{Inpainting (top) and $2\mbox{x}$-superresolution (bottom) using the symmetric DSG-NLM denoiser \cite{sreehari2016plug}. PSNR: (a) $23.23\,\si{dB}$, (b) $30.22\,\si{dB}$, (c) $30.25\,\si{dB}$; (d) $8.40\,\si{dB}$, (e)  $25.56\,\si{dB}$,  (f) $25.57\,\si{dB}$.}
\label{fig:reconstructions}
\end{figure}

\begin{figure}
\centering
\subfloat[$\y_0 =\,\mbox{all-zero image}$]{\includegraphics[width=0.32\linewidth]{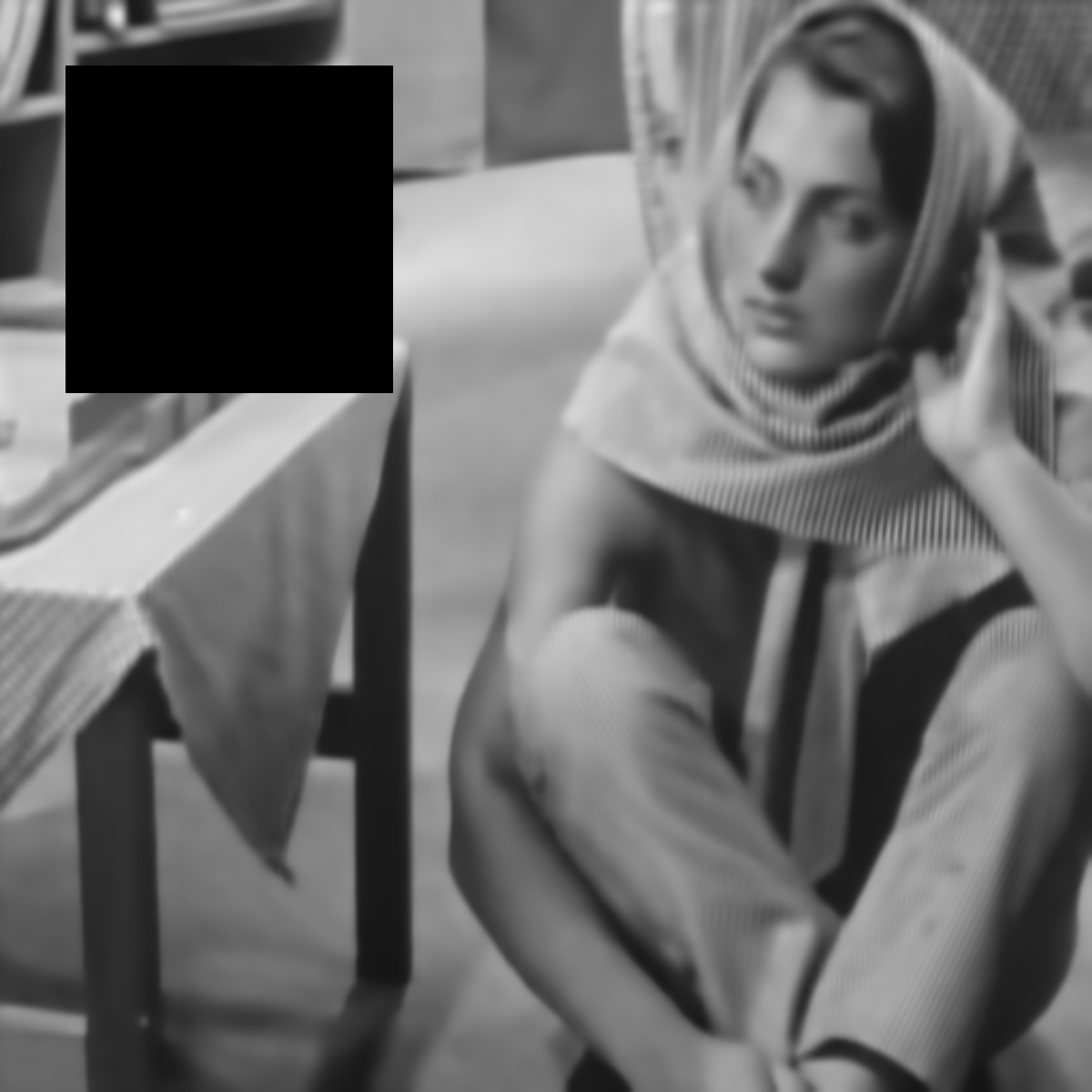}} \hspace{0.1mm}
\subfloat[$\y_0 = \,\mbox{all-one image}$]{\includegraphics[width=0.32\linewidth]{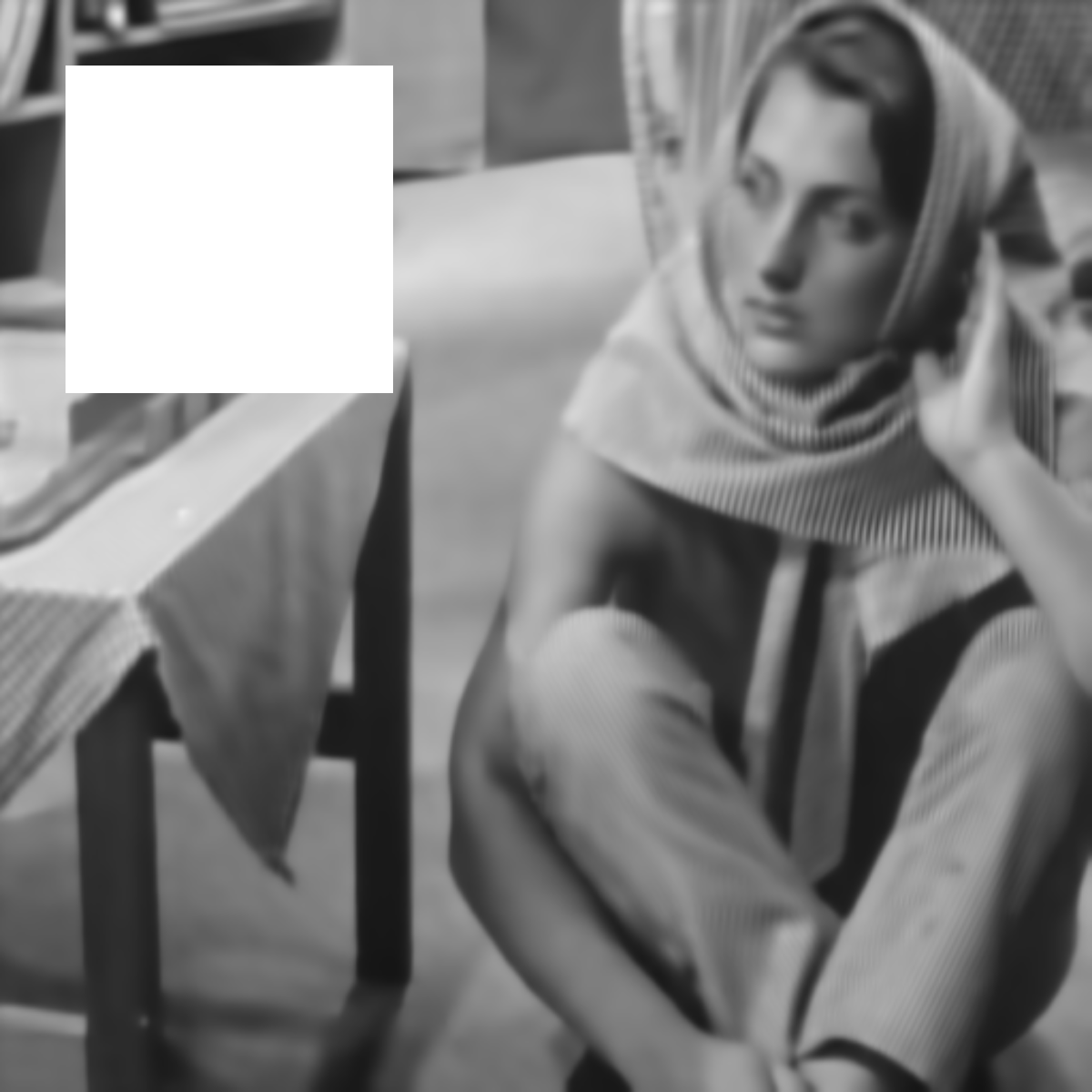}} \hspace{0.1mm}
\subfloat[$\y_0 = \,\mbox{white noise}$]{\includegraphics[width=0.32\linewidth]{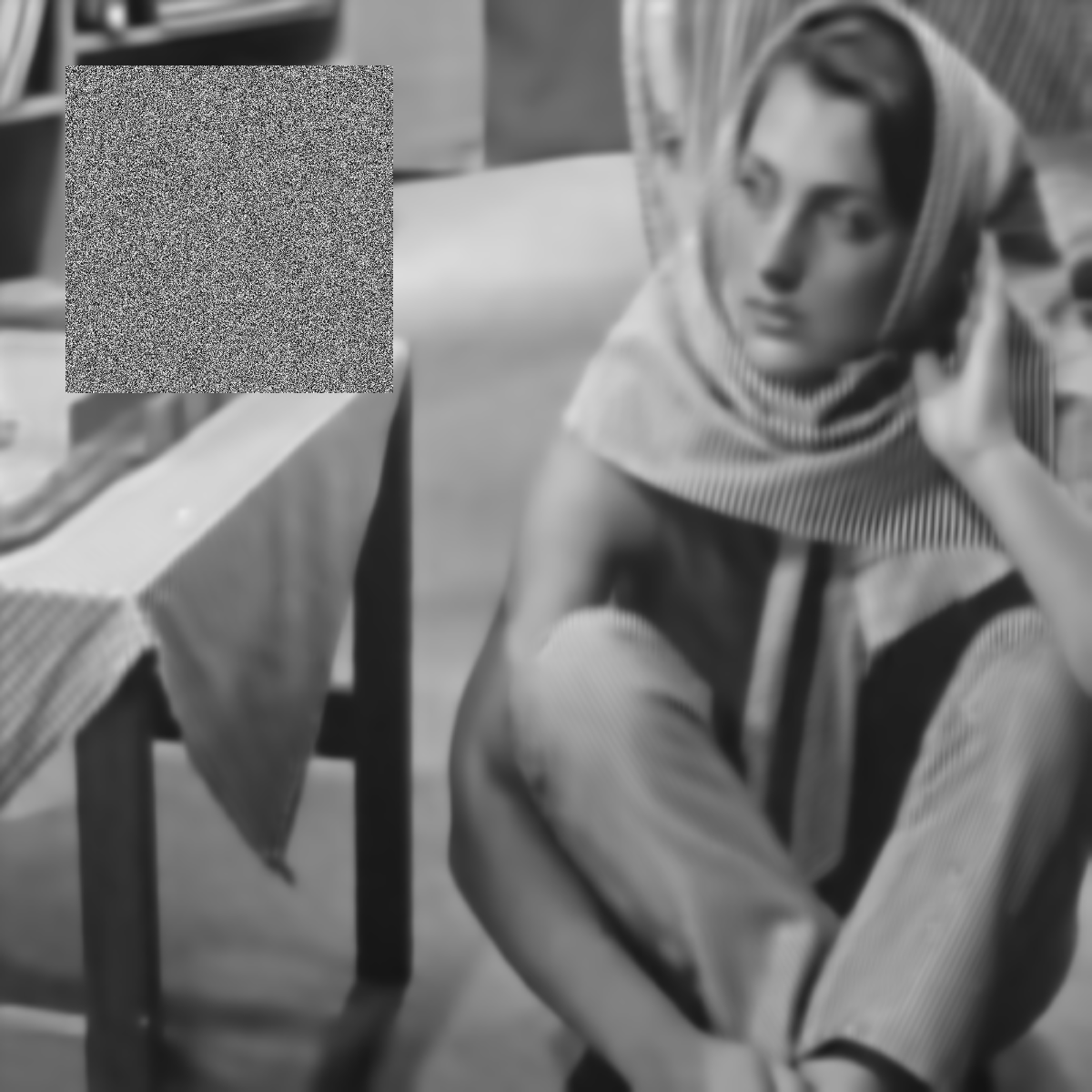}}
\caption{Deblurring results for PnP-ADMM using the symmetric GMM denoiser \cite{Teodoro2019PnPfusion}. PSNR for three different initializations (see insets):  (a) $24.04\,\si{dB}$, (b) $24.04\,\si{dB}$, (c) $24.04\,\si{dB}$.}
\label{fig:three-initializations}
\end{figure}

Fig.~\ref{fig:plots} validates the contractivity of PnP-ISTA and PnP-ADMM for inpainting, deblurring, and superresolution with a symmetric denoiser (Corollaries \ref{cor:PnP-ISTA-applications} and \ref{cor:PnP-ADMM_linear-convergence}). Furthermore, it also validates the contractivity of PnP-ISTA and PnP-ADMM for inpainting using a kernel denoiser (Theorem \ref{thm:PnP-kernel}).

In Fig.~\ref{fig:convergence}, we demonstrate linear convergence of PnP-ISTA with a symmetric denoiser, where the experiments have been performed on the Barbara image. As asserted in Corollary \ref{cor:PnP-ISTA-applications}, the residual $\|\x_k-\x^*\|_2$ is seen to monotonically converge to zero. We also observed similar behavior for PnP-ADMM with a symmetric denoiser.

In Fig.~\ref{fig:reconstructions}, we show the reconstructions for inpainting and superresolution using the symmetric DSG-NLM denoiser. The initialization $\x_0$ is obtained by median filtering the observed image $\b$ for inpainting and by median filtering the adjoint $\A^\top  \b$ for superresolution. We used $100$ iterations for both experiments. We notice that the reconstructions are of high quality, as already reported in the literature \cite{sreehari2016plug,zhang2017learning, CWE2017}. 

We have shown that PnP-ISTA and PnP-ADMM using symmetric denoisers are contractive for inpainting, deblurring, and superresolution.
Therefore, it follows from the contraction mapping theorem that for an arbitrary initialization, the PnP iterates should converge to the unique fixed point of the PnP operator. 
This is demonstrated in Fig.~\ref{fig:three-initializations} for deblurring using the symmetric GMM denoiser \cite{Teodoro2019PnPfusion}. Starting from three different initializations, PnP-ADMM is observed to converge to the same reconstruction; PnP-ISTA also exhibits similar behavior.
 
\section{Conclusion}
\label{sec:conc}
We provided new insights into the convergence of PnP for linear inverse problems. In particular, using the contraction mapping theorem, we established linear convergence of PnP-ISTA and PnP-ADMM with symmetric denoisers. We numerically validated our results on inpainting, deblurring, and superresolution. Furthermore, we established linear convergence of PnP-ISTA and PnP-ADMM for inpainting with nonsymmetric kernel denoiser.

We can establish contractivity under additional assumptions on $\A$ and $\W$ for deblurring and superresolution with kernel denoisers. However, these additional assumptions are rarely met in practice, so we did not report these results.
Establishing convergence for deblurring and superresolution with nonsymmetric kernel denoisers thus remains an open problem. 

\appendix
\section{Appendix}

\textbf{Notation}: $\Sy^n$ denotes the space of $(n \times n)$ symmetric matrices; $\Sy^n_+$ and $\Sy^n_{++}$ denote the set of $(n \times n)$ symmetric positive semidefinite and positive definite matrices; and $\sigma(\T)$ denotes the spectrum of $\T$.
\subsection{Proof of Lemma \ref{lemma:contractivity}}
\begin{remark}
\label{rmk:invariant-space}
A subspace $\Xr \subseteq \Re^n$ is said to be $\T$-invariant if $\T\x \in \Xr$ for all $\x \in \Xr$. 
Notice that in Lemma \ref{lemma:contractivity}, $\Ur$ (resp. $\Vr$) is $\M$-invariant (resp. $\N$-invariant). \fd
\end{remark}
The following proposition is used in the proof of Lemma \ref{lemma:contractivity}. 
\begin{proposition}
\label{prop:contractivity}
Let $\T \in \Sy^n$ be such that $\|\T\|_2 \leqslant 1$. Let $\Yr \subseteq \Re^n$ be the space spanned by eigenvectors of $\T$ corresponding to eigenvalues $\pm 1$. 
Then $\|\T\z\|_2 < \|\z\|_2$ for all $\z \in \Re^n \setminus \Yr$.
\end{proposition}

\begin{proof}
Since $\|\T\|_2 \leqslant 1$, $|\lambda| \leqslant 1$ for all $\lambda \in \sigma(\T)$.   
Let $\Xr \subseteq \Re^n$ be the subspace spanned by eigenvectors of $\T$ corresponding to $\lambda \in \sigma(\T)$ such that $|\lambda| < 1$.
%i.e., 
%\begin{equation*}
%\Xr = \Span\, \left\{\x \in \Re^n \setminus \{\ze\} : \T\x = \lambda \x \mbox{ and } |\lambda| < 1 \right\}
%\end{equation*}
Note that $\Xr$ and $\Yr$ are $\T$-invariant subspaces; moreover, since $\T \in \Sy^n$, $\Xr^\perp = \Yr$. 
Therefore, for any arbitrary $\z \in \Re^n \setminus \Yr$, there exist unique $\x \in \Xr \setminus \{\ze\}$ and $\y \in \Yr$ such that $\z=\x+\y$. 
Now, since $\Xr$ and $\Yr$ are orthogonal and $\T$-invariant, we have
\begin{equation}
\label{eq:pythagoras}
\|\z\|_2^2 = \|\x\|_2^2 + \|\y\|_2^2 \quad\mbox{and}\quad \|\T\z\|_2^2 = \|\T\x\|_2^2 + \|\T\y\|_2^2.
\end{equation}
From the definition of $\Yr$, we have $\|\T\y\|_2 = \|\y\|_2$.
Also, since $\x \neq \ze$, it follows from the construction of $\Xr$ that $\|\T\x\|_2 < \|\x\|_2$. 
Thus, using \eqref{eq:pythagoras}, we have $\|\T\z\|_2 < \|\z\|_2$. Since $\z$ is arbitrary, this completes the proof.
\end{proof}

\begin{remark}
In Proposition \ref{prop:contractivity}, note that $\|\T\|_2 < 1$ if and only if $\Yr = \{\ze\}$. 
While in general, in the context of Proposition \ref{prop:linear-convergence}, $\|\Q\| < 1$ if and only if $\|\Q\x\| < \|\x\|$ for all $\x \in \Re^n \setminus \{\ze\}$.  \fd
\end{remark}

\begin{proof}[Proof of Lemma \ref{lemma:contractivity}]
$(\Leftarrow)$: For any arbitrary $\z \in \Re^n \setminus \{\ze\}$, there are two possibilities: (i) $\z \in \Vr$, (ii) $\z \notin \Vr$. 
In case-(i), since $\z \neq \ze$, it follows from the definition of $\Vr$ that $\N\z \in \Vr \setminus \{\ze\}$; in fact, $\|\N\z\|_2=\|\z\|_2$. 
Now, since $\Ur \cap \Vr = \{\ze\}$, we can conclude that $\N\z \notin \Ur$. 
Subsequently, by Proposition \ref{prop:contractivity}, we have $\|\M\N\z\|_2 < \|\N\z\|_2 = \|\z\|_2$.
In case-(ii), it follows from the hypothesis and Proposition \ref{prop:contractivity} that $\|\M\N\z\|_2 \leqslant \|\N\z\|_2 < \|\z\|_2$. 
This shows that $\|\M\N\z\|_2 < \|\z\|_2$ for all $\z \in \Re^n \setminus \{\ze\}$; thus, $\|\M\N\|_2 < 1$.

$(\Rightarrow)$: Suppose there exists $\ze \neq \x \in \Ur \cap \Vr$. Since $\x \in \Vr$, it follows from the definition of $\Vr$ that $\x = \y+\z$ where $\y \in \cN(\I-\N)$ and $\z \in \cN(\I+\N)$; furthermore, as $\N \in \Sy^n$, we have $\|\x\|^2=\|\y\|_2^2 + \|\z\|_2^2$. Let $\w := \y-\z$; notice that $\N\w=\x$. 
%%%%% $\N\u = \x \neq \ze$ %%%%%
Moreover, since $\N\w=\x \in \Ur$ and $\|\M\u\|_2 = \|\u\|_2$ for all $\u \in \Ur$, we get  
\begin{equation*}
\|\M(\N\w)\|_2^2 = \|\N\w\|_2^2 = \|\x\|_2^2 = \|\y\|_2^2 + \|\z\|_2^2 = \|\w\|_2^2,
\end{equation*}
This contradicts $\|\M\N\|_2 < 1$.
\end{proof}

\subsection{Proof of Theorem \ref{thm:PnP-ISTA_contractivity}}
\begin{proof}
Note that since $\A^\top \! \A \in \Sy^n_+$, it follows from the spectral mapping theorem that if $0 < \gamma < 2/\varrho(\A^\top \! \A)$, then $\sigma(\G) \subset (-1,1]$, where $\G$ is given by \eqref{eq:def-PG}	. Therefore, $\G$ being symmetric, we have $\|\G\|_2 \leq1 $. Recall that $\|\W\|_2 = 1$. 
We now apply Lemma \ref{lemma:contractivity} with $\M=\W$ and $\N=\G$; notice that in this case, $\Ur = \cN(\I-\W)$ and $\Vr = \cN(\A^\top \! \A) = \cN(\A)$.
\end{proof}

\subsection{Proof of Theorem \ref{thm:PnP-ADMM_contractivity}}
\begin{remark}
\label{rmk:F-&-V-spectrum}
Notice from \eqref{eq:def-VF} that $\F \in \Sy^n$. We can conclude using the spectral mapping theorem that $\sigma(\F) \subseteq (-1, 1]$ for all $\rho > 0$; consequently, $\|\F\|_2 \leqslant 1$. Also, since $\sigma(\W) \subseteq [0,1]$ and $\W \in \Sy^n$, we have $\sigma(\V) \subseteq [-1, 1]$ and $\|\V\|_2 \leqslant 1$. \fd
\end{remark}

\begin{remark}
\label{rmk:F-&-V-contractivity}
In the context of Proposition \ref{prop:contractivity}, the corresponding subspaces for $\F$ and $\V$ are $\cN(\A)$ and $\cN(\I-\W) \oplus \cN(\W)$, respectively.
%%%%%%%%%% Above claims are proved in the supplement of SPL (Proof of Theorem 2) %%%%%%%%%%
Thus, $\|\F\z\|_2 < \|\z\|_2$ for all $\z \in \Re^n \setminus \cN(\A)$, and $\|\V\z\|_2 < \|\z\|_2$ for all $\z \in \Re^n \setminus \big(\cN(\I-\W) \oplus \cN(\W)\big)$. \fd
\end{remark}
\begin{proof}
$(\Rightarrow)$: Suppose there exists $\ze \neq \x \in \cN(\I-\W) \cap \cN(\A)$. 
From  $\x \in \cN(\I-\W)$, we get $\V\x = \x$; and from $\x \in \cN(\A)$, we have $\F\x= \x$.
Consequently, using \eqref{eq:def-R}, we get $\J\x=\x$ and $\R\x=\x$, which contradicts $\|\R\|_2 < 1$.

$(\Leftarrow)$: Let $\Wr := \cN(\I-\W) \oplus \cN(\W)$. For any $\x \in \Re^n \setminus \{\ze\}$, we consider the following exhaustive cases.  
\begin{description}
\item[Case $(\x \notin \Wr).$] It follows from Remarks \ref{rmk:F-&-V-spectrum} and \ref{rmk:F-&-V-contractivity} that 
\begin{equation*}
\|\J\x\|_2 \leqslant \|\V\x\|_2 < \|\x\|_2.
\end{equation*}
Therefore, using \eqref{eq:def-R}, $\|\R\x\|_2 < \|\x\|_2$. 

\item[Case $(\x \in \Wr).$] This means $\x=\y+\z$ where $\y \in \cN(\I-\W)$ and $\z \in \cN(\W)$. There are two subcases: (i) $\V\x \notin \cN(\A)$, and (ii) $\V\x \in \cN(\A)$. In subcase-(i), it follows from \eqref{eq:def-R} and Remark \ref{rmk:F-&-V-contractivity} that $\|\J\x\|_2 < \|\V\x\|_2 \leqslant \|\x\|_2$;
%\begin{equation*}
%\|\J\x\|_2 < \|\V\x\|_2 \leqslant \|\x\|_2.
%\end{equation*}
thus, $\|\R\x\|_2 < \|\x\|_2$. 

In subcase-(ii), since $\cN(\I-\W) \cap \cN(\A) = \{\ze\}$, we claim that $\z \neq \ze$. Note that $\V\x= \y-\z$. Thus, if $\z = \ze$, we have $\V\x=\y \in \cN(\I-\W) \cap \cN(\A)$ which contradicts $\cN(\I-\W) \cap \cN(\A) = \{\ze\}$.
Notice that since $\V\x \in \cN(\A)$, we have $\F(\V\x)=\V\x$. Subsequently, using \eqref{eq:def-R} and $\V\x = \y-\z$,  we get
\begin{equation}
\label{eq:Rx}
\R\x = \frac{1}{2}(\x + \J \x) =  \frac{1}{2}(\x + \V \x) = \y.
\end{equation}
Since $\W \in \Sy^n$, we have $\|\x\|^2=\|\y\|_2^2 + \|\z\|_2^2$. Therefore, it follows from \eqref{eq:Rx} and $\z \neq \ze$ that $\|\R\x\|_2 < \|\x\|_2$.
\end{description}
We have shown that $\|\R\x\|_2 < \|\x\|_2$ for all $\x \in \Re^n \setminus \{\ze\}$; hence, $\|\R\|_2 < 1$.
\end{proof}

\begin{remark}
\label{rmk:PnP-ADMM_linear-convergence}
If $\|\R\|_2 < 1$, then it follows from Proposition \ref{prop:linear-convergence} that $\{\u_k\}$ in \eqref{eq:uk} convergences linearly to some $\u^* \in \Re^n$. 
This indeed implies that $\{\y_k\}$ converges linearly to $\y^*:=\W\u^*$; note that since $\y_k = \W\u_k$ and $\|\W\|_2 = 1$, we have $\|\y_k - \y^*\|_2 \leqslant\|\u_k - \u^*\|_2$. \fd
\end{remark}

\subsection{Proof of Theorem \ref{thm:PnP-kernel}}
\begin{proof}
It follows from Remark \ref{rmk:D-norm} that $\|\P\|_\D= \|\D^\frac{1}{2} \P \D^{-\frac{1}{2}}\|_2$, where $\D^\frac{1}{2} \in \Sy^n_{++}$ is the symmetric square root of $\D$. 
It follows from \eqref{eq:def-PG} and \eqref{eq:A-inpaint} that $\G$ is diagonal for inpainting.
Subsequently, using \eqref{eq:def-PG}, we get
\begin{equation}
\label{eq:P-simplification}
\D^\frac{1}{2} \P \D^{-\frac{1}{2}} = \big(\D^\frac{1}{2} \W \D^{-\frac{1}{2}}\big) \big(\D^\frac{1}{2} \G \D^{-\frac{1}{2}}\big) = \W_0 \G,
\end{equation}
where $\W_0 := \D^\frac{1}{2} \W \D^{-\frac{1}{2}}$. 
Notice that as $\W=\D^{-1}\K$ is a kernel denoiser, we have $\W_0 = \D^{-\frac{1}{2}}  \K \D^{-\frac{1}{2}} \in \Sy^n_+$. Moreover, since $\W_0$ is similar to $\W$, we have $\|\W_0\|_2=\varrho(\W_0)=\varrho(\W)=1$. 
Also, it follows from \eqref{eq:def-PG} and \eqref{eq:A-inpaint} that for inpainting, $\|\G\|_2 \leqslant 1$ if $\gamma \in (0,2)$. 
To show that $\|\D^\frac{1}{2} \P \D^{-\frac{1}{2}}\|_2 < 1$, we use Lemma \ref{lemma:contractivity} with $\M=\W_0$ and $\N=\G$. 
In this setting, $\Ur = \cN(\I-\W_0) = \Span\,\{\q\}$, where $\q := \D^\frac{1}{2} \e$, and $\Vr = \cN(\A)$. 
Since $\Omega \neq \emptyset$ and all the components of $\q$ are nonzero, it follows from \eqref{eq:A-inpaint} that $\q \notin \cN(\A)$. 
Therefore, we have from \eqref{eq:P-simplification} and Lemma \ref{lemma:contractivity} that $\|\P\|_\D = \|\W_0\G\|_2 <1$; and hence, by Proposition \ref{prop:linear-convergence}, PnP-ISTA converges linearly. 

To prove $\|\R\|_\D < 1$, it suffices to show that $\|\J\|_\D < 1$, where $\J=\F\V$ (see~\eqref{eq:def-R}). Moreover, for inpainting, $\F$ is diagonal (see~\eqref{eq:def-VF}). Therefore, using \eqref{eq:def-R}, we can write
\begin{align}
\label{eq:J-simplification}
\D^{\frac{1}{2}} \J \D^{-\frac{1}{2}} = \big(\D^{\frac{1}{2}} \F \D^{-\frac{1}{2}}\big) \big(\D^{\frac{1}{2}}\V \D^{-\frac{1}{2}}\big)  = \F \big(2\W_0 - \I\big),
\end{align}
where $\W_0 = \D^{\frac{1}{2}} \W \D^{-\frac{1}{2}} \in \Sy^n_+$ (as explained earlier). 
It follows from the expression of $\W_0$ and the assumed nonsingularity of $\W$ that $\sigma(2\W_0 - \I) \subset (-1, 1]$.
Also, recall from the proof of Theorem \ref{thm:PnP-ADMM_contractivity} that $\sigma(\F) \subset (-1, 1]$ for all $\rho > 0$. 
We now use Lemma \ref{lemma:contractivity} with $\M=\F$ and $\N=(2\W_0 - \I)$ to show that $\|\D^\frac{1}{2} \J \D^{-\frac{1}{2}}\|_2 < 1$. Recall from the proof of Theorem \ref{thm:PnP-ADMM_contractivity} that for $\M=\F$, we have $\Ur=\cN(\A)$. 
Also, notice that for $\N=(2\W_0 - \I)$, since $\sigma(2\W_0 - \I) \subset (-1, 1]$, we have $\Vr = \cN(2\W_0-2\I) = \Span\,\{\q\}$, where $\q = \D^\frac{1}{2} \e$. As explained before, $\q \notin \cN(\A)$. 
Thus, using \eqref{eq:J-simplification} and Lemma \ref{lemma:contractivity}, we have $\|\J\|_\D = \|\F(2\W_0 - \I)\|_2 < 1$; subsequently, $\|\R\|_\D < 1$. Therefore, on similar lines of Remark \ref{rmk:PnP-ADMM_linear-convergence}, we can conclude from Proposition \ref{prop:linear-convergence} that PnP-ADMM converges linearly; note that as explained earlier, $\|\W\|_\D = \|\W_0\|_2 = 1$.
\end{proof}

\bibliographystyle{unsrt}
\bibliography{refs}

\end{document}